%% file: ms.tex
\documentclass{article}
\usepackage{amsmath,amssymb,amsfonts,amsthm}
\usepackage{algorithm2e}
\usepackage{graphicx}
\usepackage{textcomp}
\usepackage{xcolor}
\usepackage{bbm}
\usepackage{fullpage}
\usepackage{lipsum}

\usepackage{hyperref}

\def\BibTeX{{\rm B\kern-.05em{\sc i\kern-.025em b}\kern-.08em
    T\kern-.1667em\lower.7ex\hbox{E}\kern-.125emX}}
    
\newcommand{\R}{{\mathbb{R}}}
\newcommand{\one}{{\mathbbm{1}}}
\newcommand{\ball}{{\mathrm{B}}}
\newcommand{\fdot}[2]{\langle #1,#2\rangle}
\newtheorem{theorem}{Theorem}
\newtheorem{lemma}{Lemma}
\newtheorem{corollary}{Corollary}

\newcommand\blfootnote[1]{%
  \begingroup
  \renewcommand\thefootnote{}\footnote{#1}%
  \endgroup
}

\begin{document}

\title{Bandit Multiclass Linear Classification for the Group Linear
  Separable Case\blfootnote{This work is first published in iSAI-NLP
    2019, Chiang Mai, Thailand.  This work is supported by the
    Thailand Research Fund, Grant RSA-6180074.  }}



\author{Jittat Fakcharoenphol\footnote{Department of Computer Engineering, Kasetsart University, Bangkok, Thailand. E-mail: jittat@gmail.com.
  }
  \and
  Chayutpong Prompak\footnote{Department of Computer Engineering, Kasetsart University, Bangkok, Thailand. E-mail: chay.promp@gmail.com}}



\maketitle

\begin{abstract}
  We consider the online multiclass linear classification under the
  bandit feedback setting.  Beygelzimer, P\'{a}l, Sz\"{o}r\'{e}nyi,
  Thiruvenkatachari, Wei, and Zhang [ICML'19] considered two notions
  of linear separability, weak and strong linear separability.  When
  examples are strongly linearly separable with margin $\gamma$, they
  presented an algorithm based on {\sc Multiclass Perceptron} with
  mistake bound $O(K/\gamma^2)$, where $K$ is the number of classes.
  They employed rational kernel to deal with examples under the weakly
  linearly separable condition, and obtained the mistake bound of
  $\min(K\cdot 2^{\tilde{O}(K\log^2(1/\gamma))},K\cdot
  2^{\tilde{O}(\sqrt{1/\gamma}\log K)})$.  In this paper, we refine
  the notion of weak linear separability to support the notion of
  class grouping, called group weak linear separable condition.  This
  situation may arise from the fact that class structures contain
  inherent grouping.  We show that under this condition, we can also
  use the rational kernel and obtain the mistake bound of $K\cdot
  2^{\tilde{O}(\sqrt{1/\gamma}\log L)})$, where $L\leq K$ represents
  the number of groups.
\end{abstract}



\input{intro}
\input{prelims}

\section{Main result}
\label{sect:main}

Our main technical result is the following margin transformation using the rational kernel.

\begin{theorem}
(Margin transformation). Let $(x_1,y_1),(x_2,y_2),\ldots,(x_T,y_T)\in \ball(0,1)\times [K]$
be a sequence of labeled examples that is group weakly linear separable with margin $\gamma >0$.
Let $L$ be number of group weakly separable such that $L\leq K.$
Let $\phi$ defined as in (\ref{eqn:phi}) let
\[
    \gamma' = \frac{\left[840\lceil\log_2(2L+2)\rceil\cdot\left\lceil\sqrt{\frac{2}{\gamma}}\right\rceil\right]^{-\frac{\lceil\log_2(2L+2)\rceil\cdot\left\lceil\sqrt{\frac{2}{\gamma}}\right\rceil}{2}}}{9\sqrt{L}},
\]
The feature map $\phi$ makes the sequence $(\phi (x_1),y_1),(\phi (x_2),y_2),\ldots,(\phi (x_T),y_T)$
strongly linearly separable with margin $\gamma'$.
\label{thm:margin-trans}
\end{theorem}

We note that the margin depends on $L$, the number of groups, instead of $K$, the number of classes.  Using Theorem~\ref{thm:margin-trans} with Theorem~\ref{thm:kernel-bandit-mistake-bound} from~\cite{BeygelzimerPSTWZ2019-separable} we obtain the following mistake bound for our algorithm.

\begin{corollary}
  (Mistake bound for group weakly linearly separable case) 
  Let $K$ be positive integer, $L\leq K$ and $\gamma$ be positive real number. 
  The mistake bound made by Algorithm~\ref{alg:kernel-bandit} when the examples are group weakly
  linearly separable with margin $\gamma$ with $L$ groups is at most
  $K\cdot 2^{\tilde{O}(\sqrt{1/\gamma}\log L)}$.
  \label{cor:mistake-bound}
\end{corollary}

Note that multiplicative factor of $K$ is hidden from the second bound of~\cite{BeygelzimerPSTWZ2019-separable} because of the $\tilde{O}$ notation on the exponent.  We cannot do that because in our exponent we have only $\log L$ which can be much smaller than $K$.  Their actual bound (showing $K$), which can be compared to ours, is $K\cdot 2^{\tilde{O}(\sqrt{1/\gamma}\log K)}$.

\subsection{Intra-group boundaries}

We first prove a structural property of intra-group classes.  The
following lemma shows that it is possible to separate one class from
the rest in the same group using only lower and upper thresholds.
This is independent of the number of classes in that group.

\begin{lemma}
  For any group $i\in [L]$, for any class $y\in G_i$, there exists
  reals $b_i\leq t_i$ such that for all $t\in[T]$ such that
  (1) when $y_t=y$, 
  \[
  b_i + \gamma \leq \langle u'_i,x_t\rangle \leq t_i - \gamma;
  \]
  and (2) when $g(y_t)=g(y)$ but $y_t\neq y$, either
  \[
  \langle x_t,u'_i\rangle \leq b_i - \gamma,
  \]
  or
  \[
  \langle x_t,u'_i\rangle \geq t_i + \gamma.
  \]
  \label{lemma:boundaries}
\end{lemma}
\begin{proof}
  Let $S_y = \{(x_j,y_j) : y_j=y, 1\leq j\leq T\}$ be the set of
  examples with label $y$.  Let $b_i=\min_{(x,y)\in S_y}\langle
  x,u'_i\rangle-\gamma$ and $t_i=\max_{(x,y)\in S_y}\langle
  x,u'_i\rangle+\gamma$.  The lemma follows from the definition of
  group weakly linear separability.
\end{proof}

\subsection{Margin transformation}

This section is devoted to the proof of
Theorem~\ref{thm:margin-trans}.  A key property of the space $\ell_2$
is that it ``contains'' all multivariate polynomials and the rational
kernel $k$ allows us to work in that space.  More specifically, by
(implicitly) transforming examples to $\ell_2$, we can use
multivariate polynomials to separate examples from different classes,
turning group weakly separability into strong linear separability in
$\ell_2$.  Therefore, to prove the margin transformation, as
in~\cite{BeygelzimerPSTWZ2019-separable}, we have to (1) establish a
separating polynomial and (2) prove the margin bound which depends on
the degree and the norm of the polynomials (defined below).

Consider a $d$-variate polynomial
$p:\R^d\rightarrow \R$ of the form
\[
p(x) =
p(x_1,x_2,\ldots,x_d) =
\sum_{\alpha_1,\alpha_2,\ldots,\alpha_d} c_{\alpha_1,\alpha_2,\ldots,\alpha_d} x_1^{\alpha_1}x_2^{\alpha_2}\cdots x_d^{\alpha_d},
\]
where the sum ranges over a finite set of $d$-tuple
$(\alpha_1,\alpha_2,\ldots,\alpha_d)$ of non-negative integers and
$c_{\alpha_1,\alpha_2,\ldots,\alpha_d}$'s are real coefficients.  We denote the degree of $p$ as $deg(p)$.
Following~\cite{KlivansS2004-halfspaces-margin}, the {\em norm of a
  polynomial $p$} is defined as
\[
\|p\| = \sqrt{\sum_{\alpha_1,\alpha_2,\ldots,\alpha_d} \left(c_{\alpha_1,\alpha_2,\ldots,\alpha_d}\right)^2}.
\]
The following lemma from~\cite{BeygelzimerPSTWZ2019-separable}
expresses this intuition precisely.

\begin{lemma}[from Lemma 9 in~\cite{BeygelzimerPSTWZ2019-separable}]
\label{lem:norm-bound}
(Norm bound) Let $p:\R^d\to \R$ be a multivariate polynomial. There exists $c\in \ell_2$ such that $p(x)=\fdot{c}{\phi(x)}_{\ell_2}$ and $\|c\|_{\ell_2}\leq 2^{deg(p)/2}\|p\|.$
\end{lemma}

As discussed previously, to prove Theorem~\ref{thm:margin-trans}, we need to show the
existence of multivariate polynomials that separate one class from the
other.  Consider class $i\in [K]$ in group $g(i)$.  Its positive
example $x$, when compared with examples from other group $j\neq
g(i)$, satisfies
\[
\langle u_{g(i)},x\rangle - \langle u_j,x\rangle
=
\langle u_{g(i)}-u_j,x\rangle 
\geq \gamma,
\]
implying that all examples in class $i$ lie in
\[
R^{+}_i = \bigcap_{j\neq g(i)}\{x : \langle u_{g(i)}-u_j,x\rangle \geq \gamma\}, 
\]
while all examples in other groups lie in
\[
R^{-}_i = \bigcup_{j\neq g(i)}\{x : \langle u_{g(i)}-u_j,x\rangle \leq -\gamma\}.
\]
When comparing with other classes $j$ in the same group $g(i)$, from
Lemma~\ref{lemma:boundaries}, we know that there exists thresholds
$b_i$ and $t_i$ that can be used to separate examples from group $i$,
i.e., all its positive examples lie in 
\[
\hat{R}^{+}_i=\{x : \langle u'_{g(i)},x\rangle \geq b_i+\gamma\}
\cap
\{x : \langle u'_{g(i)},x\rangle \leq t_i-\gamma\},
\]
while examples from other classes in group $g(i)$ lie in 
\[
\hat{R}^{-}_i=\{x : \langle u'_{g(i)},x\rangle \leq b_i-\gamma\}
\cup
\{x : \langle u'_{g(i)},x\rangle \geq t_i+\gamma\}.
\]
Let $v_b=\frac{b_i}{\|u'_{g(i)}\|}u'_{g(i)}$ and
$v_t=\frac{t_i}{\|u'_{g(i)}\|}u'_{g(i)}$.  Both sets can be expressed as
\begin{align*}
\hat{R}^{+}_i = & \ 
\{x : \langle u'_{g(i)},x\rangle \geq \langle u'_{g(i)},v_b \rangle+\gamma\} 
\ \cap 
\{x : \langle u'_{g(i)},x\rangle \leq \langle u'_{g(i)},v_t \rangle-\gamma\},
\end{align*}
while examples from other classes in group $g(i)$ lie in 
\begin{align*}
\hat{R}^{-}_i = & \
\{x : \langle u'_{g(i)},x\rangle \leq \langle u'_{g(i)},v_b \rangle-\gamma\}
\ \cup
\{x : \langle u'_{g(i)},x\rangle \geq \langle u'_{g(i)},v_t \rangle+\gamma\}.
\end{align*}

From Lemma~\ref{lem:norm-bound}, for class $i$, it is enough to
establish a multivariate polynomial $p_i$ such that
\begin{align*}
x\in R^{+}_i \cap \hat{R}^{+}_i & \ \ \ \Rightarrow & p_i(x) & \geq \gamma'/2, \\
x\in R^{-}_i \cup \hat{R}^{-}_i & \ \ \ \Rightarrow & p_i(x) & \leq -\gamma'/2.
\end{align*}

This is shown in Theorem~\ref{thm:sep-poly} below.  This theorem
is fairly technical and is proved in Section~\ref{sect:sep-poly}.

\begin{theorem}
(Polynomial approximation of intersection of halfspaces)
Let $v_1,v_2,\ldots,v_m \in V$ such that $\|v_1\|,\|v_2\|,\ldots,\|v_m\| \leq 1$.
Let $v_b,v_t\in V$ such that $\|v_b\|\leq 1$ and $\|v_t\|\leq 1$.
Let $v' \in V$ such that $\|v'\|\leq 1$.
Let $\gamma \in (0,1)$ and $x \in \ball(0,1)$.
There exists a multivariate polynomial $p:\R^d\to \R$ such that
\begin{enumerate}
\item $p(x) \geq \frac{1}{2}$ for all $x \in \left(\bigcap_{i=1}^m \left\{x: \fdot{v_i}{x}\geq \gamma\right\}\right) \cap \left\{x:\fdot{x}{v'} \geq \fdot{v_b}{v'}+\gamma\right\} \cap \left\{x: \fdot{x}{v'} \leq \fdot{v_t}{v'}-\gamma\right\},$
\item $p(x) \leq -\frac{1}{2}$ for all $x \in \left(\bigcup_{i=1}^m \left\{x: \fdot{v_i}{x}\leq -\gamma\right\}\right) \cup \left\{x: \fdot{x}{v'} \leq \fdot{v_b}{v'}-\gamma\right\} \cup \left\{x: \fdot{x}{v'} \geq \fdot{v_t}{v'}+\gamma\right\},$
\item $deg(p)=\lceil\log_2(2m+4)\rceil\cdot\left\lceil\sqrt{\frac{2}{\gamma}}\right\rceil,$
\item $\|p\|\leq \frac{9}{2}\left[420\lceil\log_2(2m+4)\rceil\cdot\left\lceil\sqrt{\frac{2}{\gamma}}\right\rceil\right]^{\frac{\lceil\log_2(2m+4)\rceil\cdot\left\lceil\sqrt{\frac{2}{\gamma}}\right\rceil}{2}}$
\end{enumerate}

\label{thm:sep-poly}
\end{theorem}

\begin{proof}[Proof of Theorem~\ref{thm:margin-trans}]

Consider class $i\in[K]$. We will apply
Theorem~\ref{thm:sep-poly}. For $j\in\{1,\ldots,L-1\}$, let
\[
v_j=\left\{
\begin{array}{ll}
  u_{g(i)}-u_j, & \mbox{if $j<g(i)$,} \\
  u_{g(i)}-u_{j+1}, & \mbox{if $j>g(i)$.}
\end{array}
\right.    
\]
Also, let $v'=u'_{g(i)}$, $v_b=\frac{b_i}{\|u'_{g(i)}\|}u'_{g(i)}$
and $v_t=\frac{t_i}{\|u'_{g(i)}\|}u'_{g(i)}$.

From Theorem~\ref{thm:sep-poly}, there exists a multivariate
polynomial $p_i:\R^d\to \R$ such that for all $t\in[T]$ and the
sequence $(x_1,y_1),(x_2,y_2),(x_t,y_t),\ldots,(x_T,y_T)$, we have
\begin{itemize}
\item if $y_t=i$, $p_i(x_t)\geq \frac{1}{2}$, since $x_t\in R^{+}_i
  \cap \hat{R}^{+}_i$, and
\item if $y_t\neq i,$ $p_i(x_t)\leq -\frac{1}{2}$, since $x_t\in
  R^{-}_i \cap \hat{R}^{-}_i$.
\end{itemize}

It is left to check the properties of $p$.
Theorem~\ref{thm:sep-poly} implies that
\[
\|p\|\leq \frac{9}{2}\left[420\lceil\log_2(2L+2)\rceil\cdot\left\lceil\sqrt{\frac{2}{\gamma}}\right\rceil\right]^{\frac{\lceil\log_2(2L+2)\rceil\cdot\left\lceil\sqrt{\frac{2}{\gamma}}\right\rceil}{2}}
\]
By Lemma~\ref{lem:norm-bound}, there exists $c_i\in\ell_2$ such that $\fdot{c_i}{\phi(x)}=p_i(x),$ and
\[
\|c_i\|_{\ell_2}\leq \frac{9}{2}\left[840\lceil\log_2(2L+2)\rceil\cdot\left\lceil\sqrt{\frac{2}{\gamma}}\right\rceil\right]^{\frac{\lceil\log_2(2L+2)\rceil\cdot\left\lceil\sqrt{\frac{2}{\gamma}}\right\rceil}{2}}.
\]

We are ready to construct strongly separable vectors for our group
weakly separable case in $\ell_2$ such that
$\|z_1\|^2+\|z_2\|^2+\ldots+\|z_L\|^2\leq 1$ and for all $t\in [T]$,
$\fdot{z_{y_t}}{x_t} \geq \gamma$, and for all $j\neq y_t$,
$\fdot{z_j}{x_t}\leq -\gamma$, by scaling $c_i$ appropriately as
follows.  We can let
\[
z_i=\frac{c_i}{\sqrt{L}\cdot \frac{9}{2}\left[840\lceil\log_2(2L+2)\rceil\cdot\left\lceil\sqrt{\frac{2}{\gamma}}\right\rceil\right]^{\frac{\lceil\log_2(2L+2)\rceil\cdot\left\lceil\sqrt{\frac{2}{\gamma}}\right\rceil}{2}}},
\]
and
\[
\gamma = \frac{\left[840\lceil\log_2(2L+2)\rceil\cdot\left\lceil\sqrt{\frac{2}{\gamma}}\right\rceil\right]^{-\frac{\lceil\log_2(2L+2)\rceil\cdot\left\lceil\sqrt{\frac{2}{\gamma}}\right\rceil}{2}}}{9\sqrt{L}},
\]
then the theorem follows.    
\end{proof}

\subsection{Separating polynomials}
\label{sect:sep-poly}

This section proves Theorem~\ref{thm:sep-poly}, i.e., we provide a
polynomial $p:\R^d\rightarrow\R$ that separates one class of examples
from the others with degree and norm bounds.

As in~\cite{BeygelzimerPSTWZ2019-separable}
and~\cite{KlivansS2004-halfspaces-margin}, we use the Chebyshev
polynomials~\cite{MasonH2002-chebyshev} $T_n(\cdot)$ defined as
follows.
\begin{align*}
    T_0(z)&=1,\\
    T_1(z)&=z,\\
    T_{n+1}&=2z T_n(z)-T_{n-1}(z) \;\;\text{for $n\geq 1$}
\end{align*}

The following two lemmas are
from~\cite{BeygelzimerPSTWZ2019-separable}.

\begin{lemma}[from Lemma 15 in~\cite{BeygelzimerPSTWZ2019-separable}]
\label{lem:cheby-prop}
(Properties of Chebyshev polynomials) Chebyshev polynomials satisfy
\begin{enumerate}
    \item $deg(T_n) = n$ for all $n \geq 0$.
    \item If $n \geq 1$, the leading coefficient of $T_n(z)$ is $2^{n-1}$.
    \item $T_n(\cos (\theta)) = \cos (n\theta)$ for all $\theta \in \R$ and all $n \geq 0$.
    \item $T_n(\cosh (\theta)) = \cosh (n\theta)$ for all $\theta \in \R$ and all $n \geq 0$.
    \item $|T_n(z)| \leq 1$ for all $z \in [-1, 1]$ and all $n \geq 0$.
    \item $T_n(z) \geq 1 + n^2(z - 1)$ for all $z \geq 1$ and all $n \geq 0$.
    \item $\|T_n\| \leq (1+\sqrt{2})^n$ for all $n \geq 0$.
\end{enumerate}
\end{lemma}

\begin{lemma}[from Lemma 14 in~\cite{BeygelzimerPSTWZ2019-separable}]
\label{lem:poly-prop}
(Properties of norm of polynomials) 
\begin{enumerate}
    \item Let $p_1,p_2,\ldots,p_n$ be multivariate polynomials and let $p(x)=\prod_{j=1}^n p_j(x)$
    be their product. Then, $\|p\|^2\leq n^{\sum_{j=1}^n deg(p_j)}\prod_{j=1}^n \|p_j\|^2$.
    \item Let $q$ be a multivariate polynomial of degree at most $s$ and let $p(x)=(q(x))^n$. Then,
    $\|p\|^2\leq n^{ns}\|q\|^{2n}$.
    \item Let $p_1,p_2,\ldots,p_n$ be multivariate polynomials. Then, 
    $\left\|\sum_{j=1}^n p_j\right\|^2 \leq n\sum_{j=1}^n \|p_j\|^2.$
\end{enumerate}
\end{lemma}

Our proof follows the approach in~\cite{BeygelzimerPSTWZ2019-separable}.

\begin{proof}[Proof of Theorem~\ref{thm:sep-poly}]
Let $r=\lceil\log_2(2m+4)\rceil$ and $s=\left\lceil\sqrt{\frac{2}{\gamma}}\right\rceil$.
Define the polynomial $p:\R^d\to \R$ as
\begin{align*}
    p(x) &= m+\frac{5}{2}-\sum_{i=1}^m (T_s(1-\fdot{v_i}{x}))^r 
    -(T_s(1-\fdot{x-v_b}{v'}/2))^r 
    -(T_s(1-\fdot{v_t-x}{v'}/2))^r.
\end{align*}

First, consider the case when
\begin{align*}
  x \in & \left(\bigcap_{i=1}^m \left\{x: \fdot{v_i}{x}\geq \gamma\right\}\right) \cap \left\{x: \fdot{x}{v'} \geq \fdot{v_b}{v'}+\gamma\right\} \cap 
  \left\{x: \fdot{x}{v'} \leq \fdot{v_t}{v'}-\gamma\right\}.
\end{align*}

Note that $\langle v_i, x\rangle \geq \gamma$ for all $i\in [m]$.
Since $\|x\|\leq 1$ and $\|v_i\|\leq 1$, we have $\fdot{v_i}{x} \in [0,1]$; 
thus, $(T_s(1-\fdot{v_i}{x}))^r \in [-1,1]$.
Consider the terms involving $v_b$ and $v_t$.  Since $\|x\|,\|v_b\|,\|v_t\|\leq 1$, we have that $\|x-v_b\|\leq 2$ and $\|v_t-x\|\leq 2$.  This implies that
$1\geq\fdot{x-v_b}{v'}/2\geq\gamma/2$ and $1\geq\fdot{v_t-x}{v'}/2\geq\gamma/2$; hence,
$(T_s(1-\fdot{x-v_b}{v'}/2))^r\in [-1,1]$ and $(T_s(1-\fdot{v_t-x}{v'}/2))^r\in[-1,1]$.
Therefore,
\[
p(x)\geq m+\frac{5}{2}-m-1-1\geq \frac{1}{2}.
\]

Now consider the case when
\begin{align*}
  x  \in & \bigcup_{i=1}^m \left\{x:\fdot{v_i}{x}\leq -\gamma\right\} \cup \left\{x:\fdot{x}{v'} \leq \fdot{v_b}{v'}-\gamma\right\} \cup 
  \left\{x:\fdot{x}{v'} \geq \fdot{v_t}{v'}+\gamma\right\}
\end{align*}
There are two subcases to consider.

{\em Subcase 1:} Suppose that for some $i$, $\fdot{v_i}{x}\leq-\gamma$.  
In this case, $1-\fdot{v_i}{x}\geq 1+\gamma$ and
Lemma~\ref{lem:cheby-prop} (part 6) implies that
\[
  T_s(1-\fdot{v_i}{x}) \geq 1+s^2\gamma \geq 1+2 \geq 2,
\]
and thus, $(T_s(1-\fdot{v_i}{x}))^r\geq 2^r\geq 2m+4$.

Since $T_s(1-\fdot{v_i}{x}))^r \geq -1$ for all $i$, 
$(T_s(1-\fdot{x-v_b}{v'}/2))^r\geq -1$, and $(T_s(1-\fdot{v_t-x}{v'}/2))^r\geq -1$, 
we have that
\begin{align*}
    p(x)&=m+\frac{5}{2}-(T_s(1-\fdot{v_i}{x}))^r 
    -\sum_{j\in [m]j\neq i} (T_s(1-\fdot{v_j}{x}))^r \\
    & \; \; \; \; \; -(T_s(1-\fdot{x-v_b}{v'}/2))^r 
    -(T_s(1-\fdot{v_t-x}{v'}/2))^r \\
    &\leq m+\frac{5}{2}-(2m+4)+(m-1)+2 \leq -\frac{1}{2}.
\end{align*}

{\em Subcase 2:} Consider the other case when for all $i$, $\fdot{v_i}{x} > -\gamma$.
We deal with the case that $\fdot{x}{v'} \leq \fdot{v_b}{v'}-\gamma$.  
The case when $\fdot{x}{v'} \geq \fdot{v_t}{v'}+\gamma$ can be handled similarly.

Since $\fdot{x-v_b}{v'} \leq -\gamma$, we have $1-\fdot{x-v_b}{v'}/2\geq 1 + \gamma/2$.
Lemma~\ref{lem:cheby-prop} (part 6) implies that
\[
  T_s(1-\fdot{x-v_b}{v'}/2)\geq 1+s^2\gamma/2 \geq 1+2/2 \geq 2,
\]
and $(T_s(1-\fdot{x-v_b}{v'}/2))^r\geq 2m+4$.  
Applying the same argument as in Subcase 1, this implies that $p(x)\leq-\frac{1}{2}$.

The degree of $p$ is the maximum degree of the terms $(T_s(1-\fdot{v_i}{x}))^r$, $(T_s(1-\fdot{x-v_b}{v'}/2))^r$, and $(T_s(1-\fdot{v_t-x}{v'}/2))^r$; thus, it is $r\cdot s$.

Finally, we prove the upper bound of norm of $p$.  We first deal with the term 
$T_s(1-\fdot{v_i}{x})$.

Let $f_i(x)=1-\fdot{v_i}{x}$ and $g_i(x)=T_s(1-\fdot{v_i}{x})=T_s(f_i(x))$.
We have
\[
\|f_i\|^2=1+\|v_i\|^2\leq 1+1=2.
\]
Let $T_s(z)=\sum_{j=0}^s c_j z^j$ be the expansion of $s$-th Chebyshev polynomial.
We can bound the term $\|g_i\|^2$ as follows.
\begin{align*}
    \|g_i\|^2  
    &=\left\|\sum_{j=0}^s c_j(f_i)^j\right\|^2 \\
    &\leq (s+1)\sum_{j=0}^s\left\| c_j(f_i)^j\right\|^2 \qquad\qquad\qquad \mbox{(by part 3 of Lemma~\ref{lem:poly-prop})}\\
    &=(s+1)\sum_{j=0}^s c_j^2\left\|(f_i)^j\right\|^2 \\
    &\leq (s+1)\sum_{j=0}^s c_j^2j^j\left\|f_i\right\|^{2j} \qquad\qquad\qquad \mbox{(by part 2 of Lemma~\ref{lem:poly-prop})}\\
    &\leq (s+1)\sum_{j=0}^s c_j^2j^j2^{2j} \\
    &\leq (s+1)s^s2^{2s}\sum_{j=0}^s c_j^2 \\
    &=(s+1)s^s2^{2s}\|T_s\|^2 \\
    &=(s+1)s^s2^{2s}(1+\sqrt{2})^{2s}  \qquad\qquad\qquad \mbox{(by part 7 of Lemma~\ref{lem:cheby-prop})}\\
    &=(s+1)\left(4(1+\sqrt{2})^2s\right)^s \\
    &\leq (8(1+\sqrt{2})^2s)^s \qquad\qquad\qquad\qquad\quad \mbox{(because $s+1\leq2^s$)} \\
    &\leq (47s)^s.
\end{align*}

We now deal with the terms $(T_s(1-\fdot{x-v_b}{v'}/2))^r$, and
$(T_s(1-\fdot{v_t-x}{v'})/2)^r$.

Let $h_b(x)=1-\fdot{x-v_b}{v'}/2$ and $h_t(x)=1-\fdot{v_t-x}{v'}/2$.
Let $q_b(x) = T_s(1-\fdot{x-v_b}{v'}/2) = T_s(h_b(x))$ and
$q_t(x) = T_s(1-\fdot{v_t-x}{v'}/2) = T_s(h_t(x))$.
We have
\[
\|h_b\|^2 \leq \left\|\frac{v'}{2}\right\|^2 + \left(1 + \frac{\|v_b\|\|v'\|}{2}\right)^2 \leq \frac{1}{4} + \left(1 + \frac{1}{2}\right)^2 = \frac{10}{4}\leq 3,
\]
and
\[
\|h_t\|^2 \leq \left\|\frac{v'}{2}\right\|^2 + \left(1 + \frac{\|v_t\|\|v'\|}{2}\right)^2 \leq \frac{1}{4} + \left(1 + \frac{1}{2}\right)^2 = \frac{10}{4}\leq 3,
\]
since
$h_b(x) = \fdot{x}{v'/2} + (1 + \fdot{v_b}{v'/2})$
and
$h_t(x) = -\fdot{x}{v'/2} + (1 - \fdot{v_t}{v'/2})$.

The terms $\|q_b\|^2$ and $\|q_t\|^2$ can be analyzed similarly as $\|g_i\|^2$.  We have that
\begin{align*}
    \|q_b\|^2 
    &=\left\|\sum_{j=0}^s c_j(h_b)^j\right\|^2 \\
    &\leq (s+1)\sum_{j=0}^s c_j^2j^j\left\|h_b\right\|^{2j} \qquad\qquad\qquad \mbox{(by parts 2 and 3 of Lemma~\ref{lem:poly-prop})}\\
    &\leq (s+1)\sum_{j=0}^s c_j^2j^j3^{2j} \\
    &\leq (s+1)s^s3^{2s}\sum_{j=0}^s c_j^2 \\
    &=(s+1)s^s3^{2s}\|T_s\|^2 \\
    &=(s+1)s^s3^{2s}(1+\sqrt{2})^{2s} \qquad\qquad\qquad \mbox{(by part 7 of Lemma~\ref{lem:cheby-prop})}\\
    &=(s+1)\left(9(1+\sqrt{2})^2s\right)^s \\
    &\leq (9(1+\sqrt{2})^2s)^s \\
    &\leq (105s)^s
\end{align*}
and
\begin{align*}
    \|q_t\|^2 
    &=\left\|\sum_{j=0}^s c_j(h_t)^j\right\|^2 \\
    &\leq (105s)^s.
\end{align*}

Finally,
\begin{align*}
    \|p\|&\leq m+\frac{5}{2}+\sum_{i=1}^m\left\|(g_i)^r\right\|
    +\left\|(q_b)^r\right\|+\left\|(q_t)^r\right\| \\
    &=m+\frac{5}{2}+\sum_{i=1}^m\sqrt{\left\|(g_i)^r\right\|^2} 
    +\sqrt{\left\|(q_b)^r\right\|^2} +\sqrt{\left\|(q_t)^r\right\|^2} \\
    &\leq m+\frac{5}{2}+\sum_{i=1}^m\sqrt{r^{rs}\left\|g_i\right\|^{2r}} 
    +\sqrt{r^{rs}\left\|q_b\right\|^{2r}}+\sqrt{r^{rs}\left\|q_t\right\|^{2r}} \\
    &\leq m+\frac{5}{2}+m r^{rs/2}(47s)^{rs/2}+r^{rs/2}(105s)^{rs/2} 
    +r^{rs/2}(105s)^{rs/2} \\
    &\leq m+\frac{5}{2}+(m+2)(105rs)^{rs/2}.
\end{align*}
Using the fact that $m\leq\frac{1}{2}2^r$ and $r,s \geq 1$, we then have
\begin{align*}
    \|p\| &\leq m+\frac{5}{2}+(m+2)(105rs)^{rs/2} \\
        &\leq \frac{1}{2}2^r+\frac{5}{2}+\left( \frac{1}{2}2^r +2\right)(105rs)^{rs/2} \\
        &\leq 2\cdot2^r+\frac{5}{2}\cdot 2^r(105rs)^{rs/2} \\
        &= 2^r\left( 2+\frac{5}{2}\right)(105rs)^{rs/2} \\
        &\leq 4^{rs/2}\cdot\frac{9}{2}(105rs)^{rs/2} \\
        &=\frac{9}{2}(420rs)^{rs/2}.
\end{align*}
Substitutions of $r$ and $s$ finish the proof.
\end{proof}

\section{Experiments}
\label{sect:exp}

While we focus mostly on the theoretical aspect of the problem, we
performed some experiment to visualize the algorithm.

We generated a dataset in $\R^2$ under the group weakly linear separable
condition, with $K=9$ classes and $L=3$ groups with margin $\gamma=0.005$,
shown in Fig~\ref{fig:data}.

\begin{figure}
\centering
\includegraphics[width=0.5\textwidth]{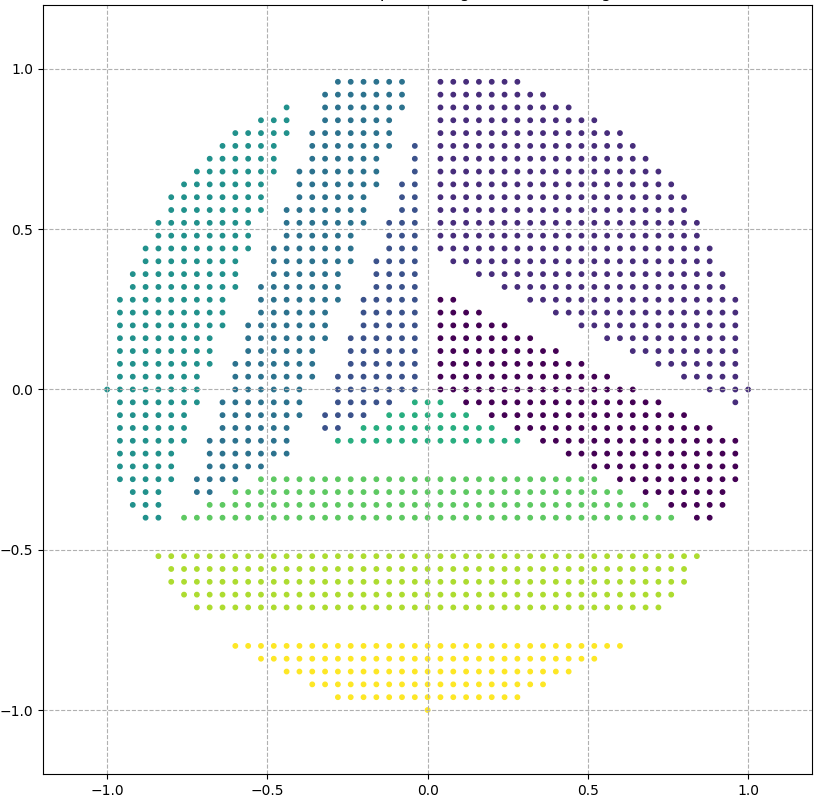}
\caption{Group weakly separable dataset in $\R^2$.}
\label{fig:data}
\end{figure}

We compared two versions of the bandit multiclass perceptron~\cite{BeygelzimerPSTWZ2019-separable},
the standard one and the kernelized one (using the rational kernel).  Since the standard one only works with strongly separable case, it would definitely fail in this experiment, but we used it to give an overall sense of improvement for the kernelized version.  We ran both algorithms for $T=10^6$ steps.  For the kernelized version, we conducted 5 experiments, while the linear one we only ran once.  Fig.~\ref{fig:result} shows the result.  The kernelized version made on average $130,884.6$ mistakes ($13.1\%$), while the standard one made $835,848$ mistakes ($83.6\%$).  Theoretically, the kernelized version should stop making mistakes at some point, but since the number of steps that we ran is too low, we can only see that increasing rate of the number of mistakes decreases over time.

\begin{figure}
\centering
\includegraphics[width=0.5\textwidth]{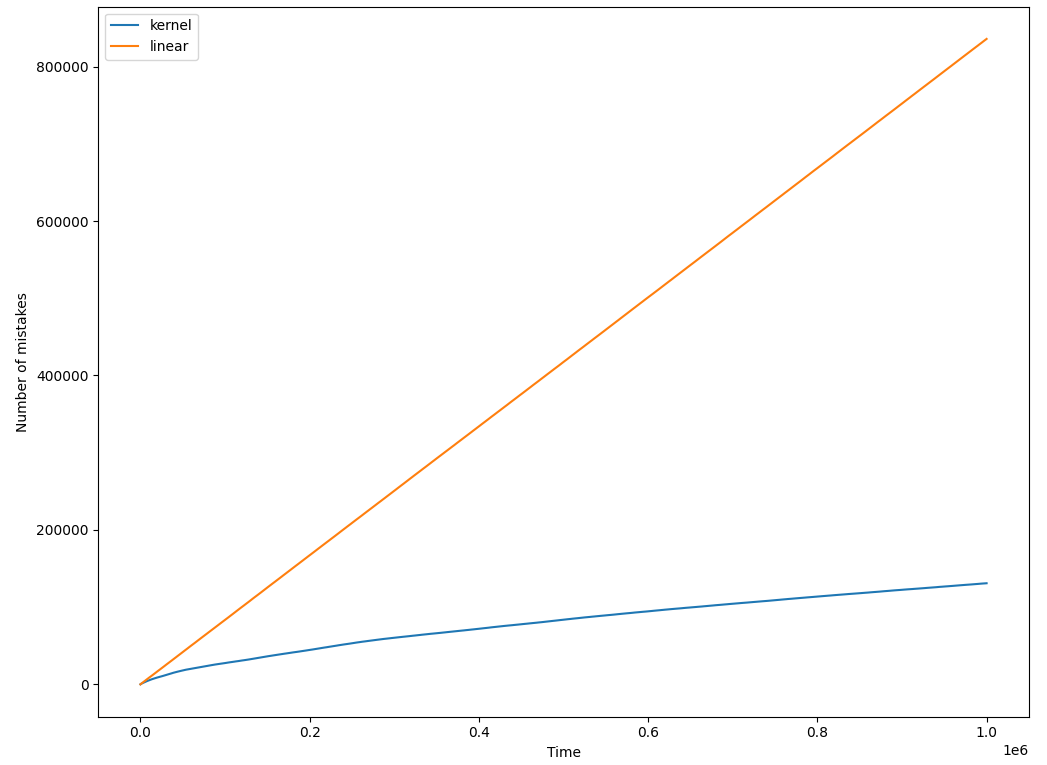}
\caption{
Comparison of the standard algorithm and the kernelized algorithm with $T=10^6$.  
}
\label{fig:result}
\end{figure}

To see the decision boundary, we ploted the contours of the corresponding polynomials for two classes shown in Fig.~\ref{fig:bound_line1} and Fig.~\ref{fig:bound_line2}.  Note that the class in Fig.~\ref{fig:bound_line2} was much harder to learn as its boundary still overlapped with other classes (i.e., mistakes could still be made).


\begin{figure}
\centering
\includegraphics[width=0.5\textwidth]{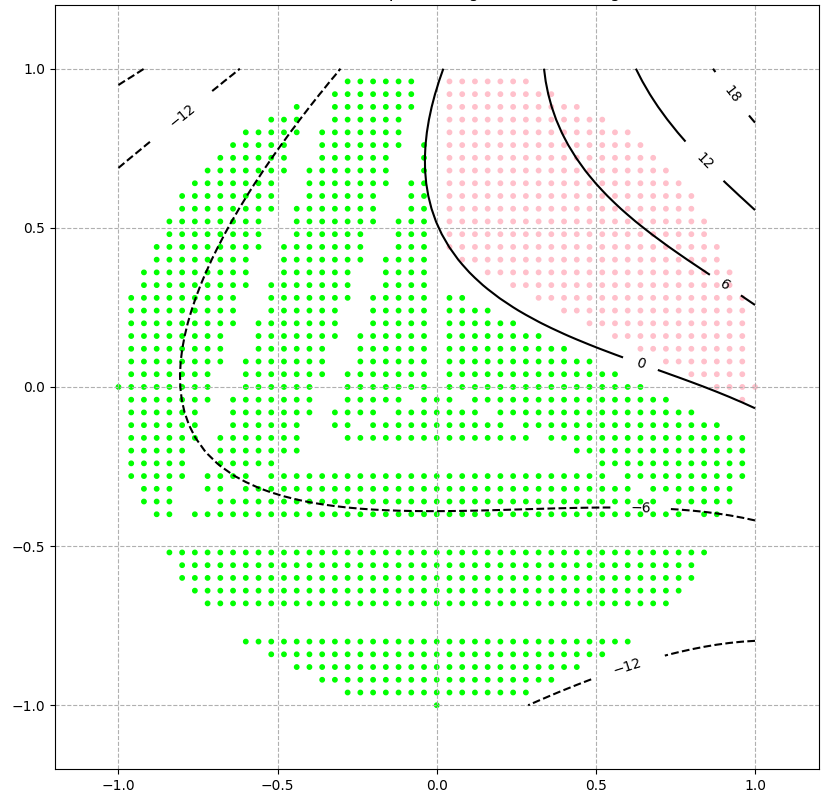}
\caption{
The decision contours of a class (in black) of the kernelized algorithm after $T=10^6$ steps. 
}
\label{fig:bound_line1}
\end{figure}

\begin{figure}
\centering
\includegraphics[width=0.5\textwidth]{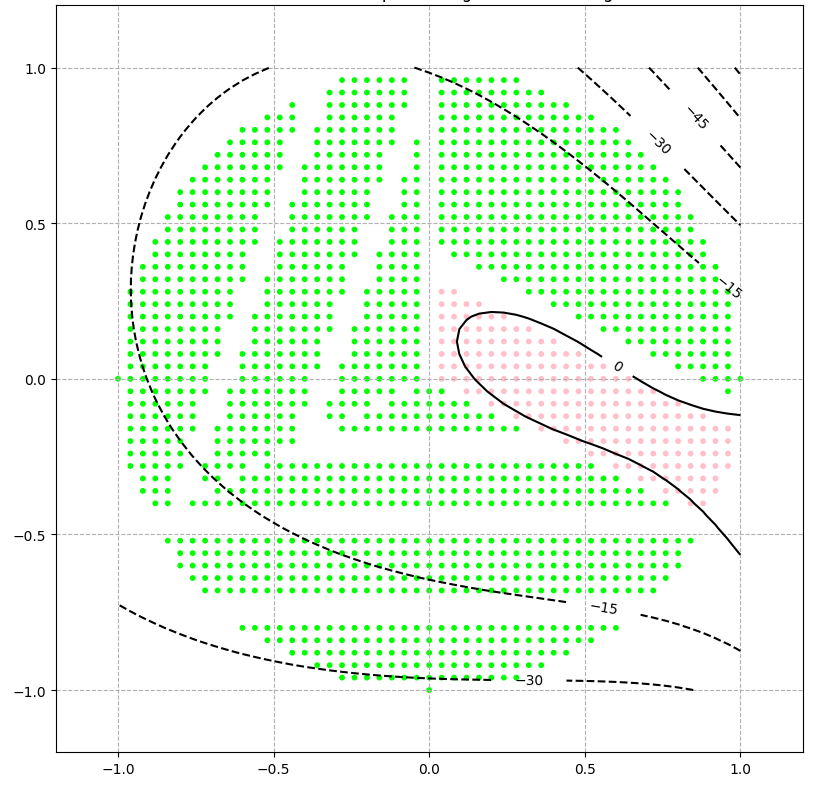}
\caption{
The decision contours of a class (in black) of the kernelized algorithm after $T=10^6$ steps.
}
\label{fig:bound_line2}
\end{figure}

\section{Acknowledgements}

We thank Sanparith Marukatat for insightful comments and for pointing out our calculation errors.  We also thank Thanawin Rakthanmanon for useful comments.

Funding: Both authors are supported by the Thailand Research Fund,
Grant RSA-6180074.

\bibliographystyle{elsarticle-num}
\bibliography{ms}

\end{document}

%% file: intro.tex
\section{Introduction}

In an online-learning paradigm, at each time step $t$, the learner
receives a feature vector $x_t$, makes a prediction $\hat{y}_t$, and
obtains a feedback.  Note that the learner is playing against an
adversary who picks the vector $x_t$ and the correct class $y_t$ from
a set of $K$ classes.  In the standard {\em full-information feedback
  setting}, the feedback is the correct class $y_t$, while in the {\em
  bandit feedback setting}, the only feedback is a binary indicator
specifying if the learner makes the correct prediction, i.e.,
$\one[\hat{y}_t=y_t]$.  The performance of the learner is measured by
the total number of mistakes over all the steps.

Typically, the theoretical analysis is carried out under particular
linear separability with margin assumptions.  Beygelzimer, P\'{a}l,
Sz\"{o}r\'{e}nyi, Thiruvenkatachari, Wei, and
Zhang~\cite{BeygelzimerPSTWZ2019-separable} introduced two definitions
of linear separability, called {\em strong} and {\em weak} linear
separability.  We give a brief summary here (see formal definitions in
Section~\ref{section:def-linear-sep}).  For both definitions, there
are $K$ vectors $w_i$ defining $K$ hyperplanes.  The weak linear
separable condition which is similar to standard multiclass linear
separability defined in Crammer and
Singer~\cite{CrammerS2003-ultraconservative} ensures that examples
from each class lie in the intersection of $K$ halfspaces induced by
these hyperplanes.  The strong linear separable condition requires
that each class is separated by a single hyperplane.

In the full-information feedback setting, Crammer and
Singer~\cite{CrammerS2003-ultraconservative} showed that if all
examples are weakly linear separable with margin $\gamma$ and have
norm at most $R$, the {\sc Multiclass Perceptron} algorithm makes at
most $\lfloor 2(R/\gamma)^2\rfloor$ mistakes.  This is tight (up to a
constant) since any algorithms must make at least $\frac{1}{2}\lfloor
(R/\gamma)^2\rfloor$ mistakes in the worst case.

For the bandit feedback
setting~\cite{KakadeST2008-bandit-online-multiclass}, Beygelzimer~{\em
  et al.}~\cite{BeygelzimerPSTWZ2019-separable} presented an algorithm
that make at most $O(K(R/\gamma)^2)$ if the examples are strongly
linear separable with margin $\gamma$, paying the price of a factor of
$K$ for the bandit feedback setting.  They also showed how to extend
the algorithm to work with weakly linear separable case using the
kernel approach.  More specifically, they showed that the examples can
be (non-linearly) transformed to higher dimensional space
so that they are strongly linear separable with margin
$\gamma'$ (which depends only on $\gamma$ and $K$).

In this paper, we introduce a more refined linear separability
condition.  Intuitively, the set of weight vectors $w_i$ represents
the ``directions'' of the examples.  In this paper, we are interested
in the cases where these directions collapse, i.e., while there are
$K$ classes of examples, the number of distinct weight vectors
required to linearly separate them is less than $K$.  This situation
may arise from the fact that class structures contain inherent
grouping where intra-group classes can be separated with a single
weight vector (or direction).  (See Fig.~\ref{fig:sep-examples}, for
example.)

More specifically, we consider the case where the classes can be
partitioned into $L$ groups, where $L\leq K$, such that (1) examples
from any two classes in the same group are linearly separable with a
margin with a single weight vector, and (2) examples from two classes
under different groups are weakly linear separable with a margin.  We
refer to this condition as the {\em group weakly linear separable
  condition}.

We show that under this refined condition, the same kernel as
in~\cite{BeygelzimerPSTWZ2019-separable} can also be used so that the
algorithm works in the space where there is (strong) margin $\gamma'$
that depends on $L$.  Our proofs, as well as that
of~\cite{BeygelzimerPSTWZ2019-separable}, use the ideas from Klivans
and Servedio~\cite{KlivansS2004-halfspaces-margin} (which is also
based on Beigel~{\em et al.}~\cite{BeigelRS1995}).

We note that our key contribution is the mathematical analysis of the
margin for group weakly linearly separable examples for the kernelized
algorithm in Beygelzimer~{\em et al.}.  This means that everything in
their paper works under this group condition (with a better margin
bound that depends on $L$ not $K$).

Section~\ref{sect:def} gives definitions and problem settings.  Our
main result is in Section~\ref{sect:main}.  In particular,
Section~\ref{sect:sep-poly} contains our technical theorem that
establishes the margin under the transformed inner product space.  We
provide small examples in Section~\ref{sect:exp}.

%% file: prelims.tex
\section{Definitions and problem settings}
\label{sect:def}

In this section, we review various definitions of linear separability
and state a new group weakly linear separable condition, the focus of
this work.  We also provide a quick review of kernel methods and the
{\sc Kernelized Bandit Algorithm} algorithm used by Beygelzimer~{\em
  et. al.}~\cite{BeygelzimerPSTWZ2019-separable}.

\subsection{Linear separability}
\label{section:def-linear-sep}

We restate the definitions for strong and weak linear separability by Beygelzimer~{\em et. al.}~\cite{BeygelzimerPSTWZ2019-separable} here.  
We use the common notation that $[K]=\{1,2,\ldots,K\}$.

The examples lie in an inner product space $(V,\langle\cdot,\cdot\rangle)$.  Let $K$ be the number of classes and let $\gamma$ be a positive real number.
Labeled examples
\[
(x_1,y_1),(x_2,y_2),\ldots,(x_T,y_T)\in V\times[K]
\]
are {\em strongly linear separable with margin $\gamma$} if there exist vectors $w_1,w_2,\ldots,w_K\in V$ such that
for all $t\in[T]$,
\[
\langle x_t, w_{y_t}\rangle \geq \gamma/2,
\]
and
\[
\langle x_t, w_i\rangle \leq -\gamma/2,
\]
for $i\in [K]\setminus \{y_t\}$,
and $\sum_{i=1}^K \Vert w_i \Vert^2\leq 1$.

On the other hand, the labeled examples are {\em weakly linear separable with margin $\gamma$} if there exist vectors $w_1,w_2,\ldots,w_K\in V$ such that
for all $t\in[T]$,
\[
\langle x_t, w_{y_t}\rangle \geq \langle x_t, w_i\rangle + \gamma,
\]
for $i\in [K]\setminus \{y_t\}$, and $\sum_{i=1}^K \Vert w_i \Vert^2\leq 1$.

The strong linear separability also appears in Chen~{\em et
  al.}~\cite{ChenCZCZ2009-beyond-banditron}.  The weak linear
separable condition appears in Crammer and
Singer~\cite{CrammerS2003-ultraconservative}.

We now define group weakly linear separability.
Let ${\mathcal G}=\{G_1,G_2,\ldots,G_L\}$ be a partition of $[K]$, i.e., $G_i\subseteq [K]$ for all $i$,
$G_i\cap G_j=\emptyset$ for $i\neq j$, and $\bigcup G_i = [K]$.  
Let $g:[K] \rightarrow [L]$ be a mapping function such that $g(i)\mapsto j$ iff $i\in G_j$.
We say that the labeled examples
\[
(x_1,y_1),(x_2,y_2),\ldots,(x_T,y_T)\in V\times[K]
\]
are {\em group weakly linear separable with margin $\gamma$ under ${\mathcal G}$} 
if 
\begin{enumerate}
\item there exist vectors $u_1,u_2,\ldots,u_L\in V$ such that
$\sum_{i=1}^L \Vert u_i \Vert^2\leq 1$, and, for all $t\in[T]$, 
\[
\langle x_t, u_{g(y_t)}\rangle \geq \langle x_t, u_p\rangle + \gamma,
\]
for all $p\in [L] \setminus\{g(y_t)\}$, 
\item there exist vectors $u'_1,u'_2,\ldots,u'_L\in V$ such that
$\sum_{i=1}^L \Vert u'_i \Vert^2\leq 1$, and, for all $t\in[T], t'\in[T]$ such that $y_t\neq y_{t'}$ and $g(y_t)=g(y_{t'})$,
either
\[
\langle x_t, u'_{g(y_t)}\rangle \geq \langle x_{t'}, u'_{g(y_t)}\rangle + 2\gamma,
\]
or
\[
\langle x_t, u'_{g(y_t)}\rangle \leq \langle x_{t'}, u'_{g(y_t)}\rangle - 2\gamma.
\]
\end{enumerate}
Note that vectors $u_i$'s define inter-group hyperplanes, while each
$u'_i$ defines intra-group boundaries.  Also note that, to simplify
our proofs, the ``margin'' between intra-group classes is $2\gamma$;
this would create the $+\gamma$ and $-\gamma$ gaps that already exist
between groups.

To illustrate the idea, Fig.~\ref{fig:sep-examples} shows 3 sets of examples.  

\begin{figure}[h]
\centering
\includegraphics[width=0.45\textwidth]{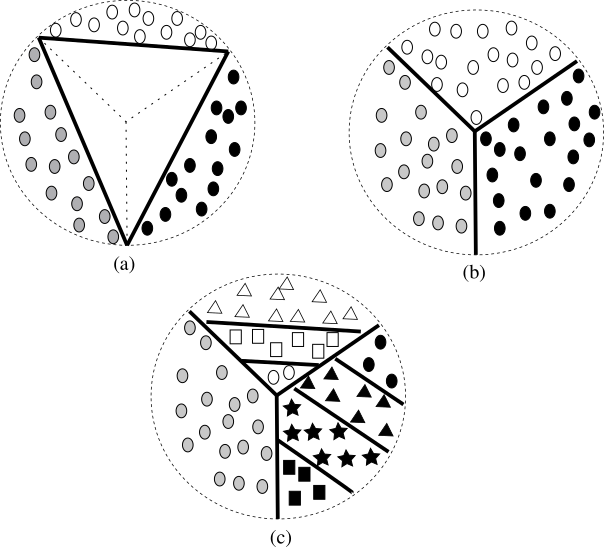}
\caption{Three set of examples in $\R^2$ showing different linear separable conditions. Thick lines represent class boundaries. (a) Strongly linear separable examples with 3 classes (linearly separable in $\R^3$).  (b) Weakly linear separable examples with 3 classes.  (c) Group weakly linear separable examples with 3 groups; group 1 (white) contains 3 classes, group 2 (black) contains 4 classes, and group 3 (gray) contains 1 class.}
\label{fig:sep-examples}
\end{figure}

\subsection{Kernel methods}

We give an overview of the kernel methods (see~\cite{ShaweTaylorC2004-kernel-methods} for expositions) and the rational kernel~\cite{ShalevShwartzSS2011-kernel-based}.

The kernel method is a standard approach to extend linear classification algorithms that use only inner products to handle the notions of ``distance'' between pairs of examples to nonlinear classification.
A {\em positive definite kernel} (or {\em kernel}) is a function of the form $k: X\times X\rightarrow \R$ for some set $X$ such that the matrix $[k(x_i,x_j)]_{i,j=1}^m$ is symmetric positive definite for any set of $m$ examples $x_1,x_2,\ldots,x_m\in X$.  It is known that for every kernel $k$, there exists some inner product space $(V,\fdot{\cdot}{\cdot})$ and a feature map $\phi:X\rightarrow V$ such that $k(x,x')=\fdot{\phi(x)}{\phi(x')}$.  Therefore, a linear learning algorithm can essentially non-linearly map every example into $V$ and work in $V$ instead of the original space without explicitly working with $\phi$ using $k$.  This can be very helpful when the dimension of $V$ is infinite.

As in Beygelzimer~{\em et al.}~\cite{BeygelzimerPSTWZ2019-separable}, we use the rational kernel.  Assume that examples are in $\R^d$.  Denote by $\ball(0,1)$ a unit ball centered at $0$ in $\R^d$.
The {\em rational kernel} $k:B(0,1)\times B(0,1)\rightarrow\R$ is defined as
\[
k(x,x')=\frac{1}{1-\frac{1}{2}\langle x,x'\rangle_{\R^d}}.
\]
Given $x,x'\in\R^d$, $k(x,x')$ can be computed in $O(d)$ time.

Let $\ell_2=\{x\in\R^{\infty} : \sum_{i=1}^\infty x_i^2 < +\infty \}$ be the classical real separable Hilbert space equipped with the standard inner product $\fdot{x}{x'}_{\ell_2}=\sum_{i=1}^\infty x_i x'_i$.  We can index the coordinates of $\ell_2$ by $d$-tuples $(\alpha_1,\alpha_2,\ldots,\alpha_d)$ of non-negative integers, the associated feature map $\phi:\ball(0,1)\rightarrow\ell_2$ to $k$ is defined as
\begin{multline}
\left(\phi(x_1,x_2,\ldots,x_d)\right)_{(\alpha_1,\alpha_2,\ldots,\alpha_d)} =
x_1^{\alpha_1}x_2^{\alpha_2}\cdots x_d^{\alpha_d}
\cdot
\sqrt{2^{-(\alpha_1+\alpha_2+\cdots+\alpha_d)}{{\alpha_1+\alpha_2+\cdots+\alpha_d}\choose{\alpha_1,\alpha_2,\ldots,\alpha_d}}},
\label{eqn:phi}
\end{multline}
where
${{\alpha_1+\alpha_2+\cdots+\alpha_d}\choose{\alpha_1,\alpha_2,\ldots,\alpha_d}}
=
\frac{(\alpha_1+\alpha_2+\cdots+\alpha_d)!}{\alpha_1!\alpha_2!\cdots\alpha_d!}$
is the multinomial coefficient.  It can be verified that $k$ is the
kernel with its feature map $\phi$ to $\ell_2$ and for any
$x\in\ball(0,1)$, $\phi(x)\in\ell_2$.

\subsection{Multiclass Linear Classification}

Beygelzimer~{\em et al.}~\cite{BeygelzimerPSTWZ2019-separable}
presented a learning algorithm for the strongly linearly separable
examples using $K$ copies of the {\sc Binary Perceptron}.  They
obtained a mistake bound of $O(K(R/\gamma)^2)$ when the examples are
from $\R^d$ with maximum norm $R$ with margin $\gamma$.

Their approach for dealing the weakly linear separable case is to use
the kernel method.  They introduced the {\sc Kernelized Bandit
  Algorithm} (Algorithm~\ref{alg:kernel-bandit}) and proved the
following theorem.

\begin{algorithm}
  \SetAlgoLined
  \DontPrintSemicolon
  \KwData{Number of classes $K$, number of rounds $T$}
  \KwData{Kernel function $k(\cdot,\cdot)$}
  \Begin{
      Initialize $J_1^{(1)}=J_2^{(2)}=\cdots=J_k^{(k)}=\emptyset$\;
      \For{$t=1,2,\ldots,T$}{
        Observe feature vector $x_t$\;
        Compute
        $S_t=\left\{i : 1\leq i\leq K, \sum_{(x,y)\in J_i^{(t)}} yk(x,x_t)\geq 0\right\}$\;
        \eIf{$S_t=\emptyset$}{
          Predict $\hat{y}_t\sim \mbox{Uniform}(\{1,2,\ldots,K\})$\;
          Observe feedback $z_t=\one[\hat{y}_t\neq y_t]$\;
          \eIf{$z_t=1$}{
            Set $J_i^{(t+1)} = J_i^{(t)}$ for all $i\in\{1,2,\ldots,K\}$
          }{
            Set $J_i^{(t+1)} = J_i^{(t)}$ for all $i\in\{1,2,\ldots,K\}\setminus\{\hat{y}_t\}$\;
            Update $J_{\hat{y}_t}^{(t+1)}=J_{\hat{y}_t}^{(t)}\cup\{(x_t,+1)\}$
          }
        }{
          Predict $\hat{y}_t\in S_t$ chosen arbitrarily\;
          Observe feedback $z_t=\one[\hat{y}_t\neq y_t]$\;
          \eIf{$z_t=1$}{
            Set $J_i^{(t+1)} = J_i^{(t)}$ for all $i\in\{1,2,\ldots,K\}\setminus\{\hat{y}_t\}$\;
            Update $J_{\hat{y}_t}^{(t+1)}=J_{\hat{y}_t}^{(t)}\cup\{(x_t,-1)\}$
          }{
            Set $J_i^{(t+1)} = J_i^{(t)}$ for all $i\in\{1,2,\ldots,K\}$
          }
        }
      }
    }
  \caption{{\sc Kernelized Bandit
      Algorithm}\label{alg:kernel-bandit}~\cite{BeygelzimerPSTWZ2019-separable}}
\end{algorithm}

\begin{theorem}[Theorem 4 from~\cite{BeygelzimerPSTWZ2019-separable}]
  Let $X$ be a non-empty set, let $(V,\fdot{\cdot}{\cdot})$ be an
  inner product space.  Let $\phi:X\rightarrow V$ be a feature map and
  let $k:X\times X\rightarrow\R$, where
  $k(x,x')=\fdot{\phi(x)}{\phi(x')}$, be the kernel.  If
  $(x_1,y_1),(x_2,y_2),\ldots,(x_T,y_T)\in X\times\{1,2,\ldots,K\}$
  are labeled examples such that
  \begin{enumerate}
  \item the mapped examples $(\phi(x_1),y_1),\ldots,(\phi(x_T),y_T)$
    are strongly linearly separable with margin $\gamma$,
  \item $k(x_1,x_1),k(x_2,x_2),\ldots,k(x_T,x_T)\leq R^2$
  \end{enumerate}
  then the expected number of mistakes that the {\sc Kernelized Bandit
    Algorithm} makes is at most $(K-1)\lfloor 4(R/\gamma)^2 \rfloor$.
  \label{thm:kernel-bandit-mistake-bound}
\end{theorem}

The key theorem for establishing the mistake bound is the following
margin transformation theorem based on the rational kernel.

\begin{theorem}[Theorem 5 from~\cite{BeygelzimerPSTWZ2019-separable}]
  (Margin transformation from~\cite{BeygelzimerPSTWZ2019-separable}). Let
  $(x_1,y_1),(x_2,y_2),\ldots,(x_T,y_T)\in \ball(0,1)\times [K]$ be a
  sequence of labeled examples that is weakly linear separable with
  margin $\gamma >0$.  Let $\phi$ defined as in (\ref{eqn:phi}) let
  \[
  \gamma_1 = \frac{
    \left[376\lceil\log_2(2K-2)\rceil\cdot\left\lceil\sqrt{\frac{2}{\gamma}}\right\rceil\right]^{
      \frac{-\lceil\log_2(2K-2)\rceil\cdot
      \left\lceil\sqrt{2/\gamma}\right\rceil}{2}}
  }{2\sqrt{K}},
  \]
  \[
  \gamma_2 =
  \frac{
    \left(2^{s+1}r(K-1)(4s+2)\right)^{-(s+1/2)r(K-1)}
  }{
    4\sqrt{K}(4K-5)2^{K-1}
  }
  \]
  where $r=2\lceil\frac{1}{4}\log_2(4K-3)\rceil+1$ and
  $s=\lceil\log_2(2/\gamma)\rceil$.  Then the feature map $\phi$ makes
  the sequence $(\phi (x_1),y_1),(\phi (x_2),y_2),\ldots,(\phi
  (x_T),y_T)$ strongly linearly separable with margin
  $\gamma'=\max\{\gamma_1,\gamma_2\}$.  Also for all $t$,
  $k(x_t,x_t)\leq 2$.
\end{theorem}

This implies the following mistake bound.

\begin{corollary}[Corollary 6 from~\cite{BeygelzimerPSTWZ2019-separable}]
  (Mistake upper bound
  from~\cite{BeygelzimerPSTWZ2019-separable}). The mistake bound made
  by Algorithm~\ref{alg:kernel-bandit} when the examples are weakly
  linearly separable with margin $\gamma$ is at most
  $\min(2^{\tilde{O}(K\log^2(1/\gamma))},2^{\tilde{O}(\sqrt{1/\gamma}\log
    K)})$.
\end{corollary}

Beygelzimer {\em et al.}~\cite{BeygelzimerPSTWZ2019-separable} gave two margin transformation proofs.  In this paper, we only provide one margin transformation based on the Chebyshev polynomials (Theorem 7 from~\cite{BeygelzimerPSTWZ2019-separable}).

\subsection{Our contribution}

We consider labeled examples with group weakly linearly separable with margin $\gamma$ and show that in this case, the rational kernel also transforms the margin and the new margin depends on the number of groups $L$ instead of the number of classes $K$.  More specifically we prove the margin transformation in Theorem~\ref{thm:margin-trans} and show the mistake bound of $K\cdot 2^{\tilde{O}(\sqrt{1/\gamma}\log L)}$ in Corollary~\ref{cor:mistake-bound}.  This can be compared to one of the mistake bound of $K\cdot 2^{\tilde{O}(\sqrt{1/\gamma}\log K)}$ in~\cite{BeygelzimerPSTWZ2019-separable}.

The proofs are fairly technical.  We follow the idea
in~\cite{BeygelzimerPSTWZ2019-separable} and construct a ``good''
polynomial that separates examples from one class to the other (strong
separation) based on the Chebyshev
polynomials~\cite{MasonH2002-chebyshev}.